\documentclass[sigconf]{acmart}

\usepackage{booktabs} % For formal tables
\usepackage{breakurl}
\usepackage{amsfonts}
\usepackage{amsmath}
\usepackage{amssymb}
\usepackage{todonotes}
\usepackage{balance}

\setcopyright{rightsretained}
\usepackage{color}

\newcommand{\cB}{\mathcal{B}}
\newcommand{\cG}{\mathcal{G}}
\newcommand{\cP}{\mathcal{P}}
\newcommand{\cR}{\mathcal{R}}
\newcommand{\cS}{\mathcal{S}}
\newcommand{\cX}{\mathcal{X}}
\newcommand{\lE}{\mathbb{E}}

\newcommand{\lP}{\mathbb{P}}

\newcommand{\lZ}{\mathbb{Z}}
\newcommand{\bw}{\overline{w}}
\newcommand{\bW}{\overline{W}}

\newcommand{\popp}{\cP_p}
\newcommand{\popt}{\cP_t}

\copyrightyear{2018}
\acmYear{2018}
\setcopyright{acmlicensed}
\acmConference[WSDM 2018]{11th ACM International
Conf. on Web Search and Data Mining}{February 5--9, 2018}{Marina Del Rey, CA, USA}
\acmPrice{15.00}
\acmDOI{10.1145/3159652.3159687}
 \acmISBN{978-1-4503-5581-0/18/02} 
 
\begin{document}
\title{Offline A/B testing for Recommender Systems}

\author{Alexandre Gilotte, Cl\'ement Calauz\`enes, Thomas Nedelec, Alexandre Abraham, Simon Doll\'e}
\affiliation{%
  \institution{Criteo Research}
  \streetaddress{32 rue Blanche 75009 Paris, France}
}
\email{f.name@criteo.com}

\iffalse
\begin{abstract}
Before A/B testing online a new version of a recommender system, it is usual to perform some offline evaluations on historical data. We focus on evaluation methods that compute an estimator of the potential uplift in revenue that could generate this new technology. It helps to iterate faster and to avoid losing money by detecting poor policies. These estimators are known as \emph{counterfactual} or \emph{off-policy estimators}. We show that traditional counterfactual estimators such as \emph{capped importance sampling} and \emph{normalised importance sampling} are experimentally not having satisfying bias-variance compromises in the context of personalised product recommendation for online advertising. We propose two variants of counterfactual estimates with different modelling of the bias that prove to be accurate in real-world conditions. We provide a benchmark of these estimators by showing their correlation with business metrics observed by running online A/B tests on a commercial recommender system.
\end{abstract}
\fi

\begin{abstract}
Online A/B testing evaluates the impact of a new technology by running it in a real production environment and testing its performance on a subset of the users of the platform. It is a well-known practice to run a preliminary offline evaluation on historical data to iterate faster on new ideas, and to detect poor policies in order to avoid losing money or breaking the system.
For such offline evaluations, we are interested in methods that can compute offline an estimate of the potential uplift of performance generated by a new technology. 
Offline performance can be measured using estimators known as \emph{counterfactual} or \emph{off-policy} estimators. 
Traditional counterfactual estimators, such as \emph{capped importance sampling} or \emph{normalised importance sampling},
exhibit unsatisfying bias-variance compromises when experimenting on personalized product recommendation systems.
To overcome this issue, we model the bias incurred by these estimators rather than bound it in the worst case, which leads us to propose a new counterfactual estimator.
We provide a benchmark of the different estimators showing their correlation with business metrics observed by running online A/B tests on a large-scale commercial recommender system.
\end{abstract}



% We no longer use \terms command

\maketitle

\section{Introduction}
\label{sec:introduction}
Personalized product recommendation has become a central part of most online marketing systems. Having efficient and reliable methods to evaluate recommender systems is critical in accelerating the pace of improvement of these marketing platforms.

Online A/B tests became ubiquitous in tech companies in order to make informed decisions on the rollout of a new technology such as a recommender system. Each new software implementation is tested by comparing its performance with the previous production version through randomised experiments. In practice, to compare two technologies, a pool of units (e.g. users, displays or servers) of the platform is split in two populations and each of them is exposed to one of the tested technologies. The careful choice of the unit reflects independence assumptions under which the test is run. At the end of the experiments, business metrics such as the generated revenue, the number of clicks or the time spent on the platform are compared to make a decision on the future of the new technology.

However, online A/B tests take time and cost money. Indeed, to gather a sufficient amount of data to reach statistical sufficiency and be able to study periodic behaviours (the signal can be different from one day to the other), an A/B test is usually implemented over several weeks. On top of this, prototypes need to be brought to production standard to be tested. These reasons prevent companies from iterating quickly on new ideas. 

To solve these pitfalls, people historically relied on offline experiments based on some rank-based metrics, such as NDCG \cite{jarvelin2000ir}, MAP \cite{baeza1999modern} or Precision@K. Such evaluations suffer from very heavy assumptions, such as independence between products or the fact that the feedback (e.g. click) can be translated into a supervised task \cite{herlocker2004evaluating, Marlin2007, pradel2012}. To overcome these limitations, some estimators were introduced \cite{bottou2013counterfactual, Li2011} to estimate offline -- i.e. using randomised historical data gathered under only one policy -- a business comparison between two systems. In the following, we shall call the procedure of comparing offline two systems based on some business metric defining the outcome an \emph{offline A/B test}.

This setting is called \emph{counterfactual reasoning} or \emph{off-policy evaluation} (OPE) (see \cite{bottou2013counterfactual} for a comprehensive study). Several estimators such as \emph{Basic Importance Sampling} (BIS, \cite{hammersley5monte, horvitz1952generalization}), \emph{Capped Importance Sampling} (CIS, \cite{bottou2013counterfactual}), \emph{Normalised Importance Sampling} (NIS, \cite{powell1966weighted}) and \emph{Doubly Robust} (DR, \cite{dudik2011doubly}) have been introduced to compute the expected reward of the tested technology $\pi_{t}$ based on logs collected on the current technology in production $\pi_{p}$. 

All theses estimators achieve different trade-offs between bias and variance. 
After explaining why BIS (Section \ref{sec:onlineofflineABtesting}) and DR (Section \ref{sec:controlVariates}) suffer from high variance in the recommendation setting, we shall focus on \emph{capped importance sampling} (Section \ref{sec:cappingWeights}). \cite{bottou2013counterfactual} proposed clipping the importance weights, which leads to a biased estimator with lower variance. However, the control of the bias is very loose in the general case and \cite{bottou2013counterfactual} only advocates modifying the source probability (the current system) to further explore whether a light clipping could be sufficient.

The main caveat of such biased estimates is that a low bias is present only under unrealistic conditions (see Section \ref{sec:cappingWeights}). Our main contribution is to propose two variants of \emph{capped importance sampling} that achieve low bias under much more realistic conditions (Section \ref{sec:stratification} and Section \ref{sec:apprTest}) and show their practical interest on real personalised recommendation systems. In Section \ref{sec:experiments}, we compare metrics observed during real online A/B tests with the pre-computed values of the different counterfactual estimators. 

\section{Settting and notation}

We consider recommender systems in the context of online product recommendation .
The task consists of displaying a set of products to a user on some e-commerce websites or on some advertising banners. These subsets of products should be personalised based on the interests of the user. 
This task is formalised as a ranking task, and not only a top-K retrieval, because the different product slots are not equivalent and exhibit different performance \cite{craswell2008}. The system outputs a ranked list of products and then maps better products to better positions. 

A recommendation policy is designed as a distribution over the top-K rankings. Good examples of distributions are the Bradley-Terry Luce (BTL) model \cite{bradley1952}, the Placket-Luce model \cite{plackett,guiver2009bayesian, cheng2010label} or more generally Thurstonian models \cite{thurstone, yellott1977}.

In the following, we will represent random variables with capital letters such as Y and realisation of random variables with lower-case letters such as y. Given a display $x$ represented by a set of contextual features as well as a set of eligible products, the recommender system outputs a probability distribution $\pi(A|X)$ where $a$ is a top-$K$ ranking on the eligible products and $K$ is the number of items that will be displayed. Taking action a in state x generates a reward $r \in [0,r_{\max}]$ that could be interpreted as a click or a purchase.

\section{Online and Offline A/B Testing}
\label{sec:onlineofflineABtesting}
In \emph{online} A/B tests, the objective is to compare two systems \emph{prod} and \emph{test} to ultimately take a decision on which one performs better than the other based on their respective business value. We consider a set of $n$ units $x$ that are randomly assigned to either \emph{prod} or \emph{test} and seek to measure the average difference in value, based on the reward signal $r \in [0,r_{\max}]$ that could be the number of clicks or the generated revenue. The choice of the units, which could be internet users, recommendation opportunities or servers, heavily depends on the independence assumptions made in order to reach a statistically significant decision in a timely manner. This assumption is called the isolation assumption \cite{bottou2013counterfactual}.

We will note the current production policy $\pi_p$ and the test policy $\pi_t$.
To compare $\pi_p$ and $\pi_t$, we estimate the average difference of value $\Delta \cR$ which is called average treatment effect and defined as
\begin{align*}
\Delta \cR(\pi_p, \pi_t) &= \lE_{\pi_t}[R] - \lE_{\pi_p}[R]
\end{align*}
where $\lE_{\pi_p}[R] = \lE[R | A] \pi_p(A | X) \lP(X)$.
During an online A/B test, units are randomly split in two populations $\popt$ and $\popp$, such that we can estimate $\Delta \cR$ using
\begin{align*}
\Delta \cR(\pi_p, \pi_t) &=\lE[R | X\in \popt] - \lE[R | X\in \popp]\,.
\end{align*}
$\Delta \cR$ is estimated by Monte-Carlo using the two datasets collected during the test $\cS_p = \{(x_i, a_i, r_i) : i \in \popp\}$ and  $\cS_t = \{(x_i, a_i, r_i) : i \in \popt\}$. We build the empirical estimator $\Delta \hat \cR$ to take the decision by performing a statistical test : 
\begin{align*}
\Delta \hat{\cR} (\pi_p, \pi_t) &= \hat \cR(\cS_t) - \hat \cR (\cS_p)
\end{align*}
where $\hat \cR(\cS)$ is the empirical average of rewards over $\cS$ gathered during the \textit{online} AB test. 

To perform an \emph{offline} A/B test, we have only one set of $n$ historical i.i.d. samples $\cS_n = \{(x_i, a_i, r_i) : i \in [n]\}$ collected using a production recommender system $\pi_{p}$ (also known as the \emph{behaviour policy} in the RL community or \emph{logging policy}).
The goal is to compare the performance of a new technology, a test policy denoted $\pi_t$, to our current system $\pi_p$ \footnote{We need a stochastic policy for $\pi_p$ that puts a non-zero probability on any ranking $a$ to prevent spurious correlations from biasing the estimators we exhibit in the following sections (see \cite{bottou2013counterfactual} for details).}. 
We can directly estimate $\lE_{\pi_p}[R]$ using $\hat \cR(\cS_n)$, but for $\lE_{\pi_t}[R]$ we cannot use a direct estimation since we do not have any data gathered under $\pi_t$.
One of the main tools to estimate the expected reward under the target policy using rewards gathered under the behaviour policy is \textit{importance sampling} or \textit{inverse propensity score} as introduced by \citet{hammersley5monte} which leads to the following Monte-Carlo estimator:
\begin{align*}
\hat{\cR}^{\rm IS}(\pi_{t}) = \frac{1}{n} \sum_{(x,a,r) \in \cS_n} \!\!\!\!w(a,x) r \,\,\,\,\,\,\,\,\,\,\,\,~\text{where}~ w(a, x) = \frac{\pi_{t}(a|x)}{\pi_{p}(a|x)}
\end{align*}

The main advantage of such an estimator that it is unbiased, while its main pitfall is usually its high variance, which depends on how different $\pi_{t}$ is from $\pi_{p}$ (this variance is unbounded).
All the difficulty here resides in the size of the action space. 
The number of top-$K$ rankings over candidate sets of size $M$ is extremely high ($K! \binom{M}{K}$), resulting in high variance of the importance weights $W$.

Several variants of importance sampling have been proposed, with different purposes, to tackle the high variance that can arise from the use of importance sampling. They use some classical variance reduction techniques, such as difference control variate or ratio control variate. However these approaches consist of using unbiased or consistent estimators and turn out to still lead to estimators suffering from high variance.
 An approach to trade off variance for bias is to clip the importance weights as proposed in \cite{bottou2013counterfactual}.
In the next section, we detail several classic methods used in importance sampling for counterfactual reasoning. 
\iffalse
-- i.e. takes stochastically actions $a \in A$ (finite) conditionally to a context $x \in X$ with a given probability $\pi_{p}(a | x)$ -- and observing in reaction a reward $r \in [0, R]$. The contexts follow a probability distribution $D$ fixed by the environment and the reward $r$ given a context $x$ and an action $a$ is defined by the distribution $q$.
\\
\\
This setting is easily transposed into the recommendation setting if we consider that the context variable $x$ gather information on the user and on the context of the recommendation. Actions $a \in A$ are the different products that can be recommended to the user.
\\
\\
In the counterfactual setting, we are interested in computing some performance metrics on a new policy $\pi_{test}$ (known as the \emph{target policy} in the RL literature). We consider metrics such as the generated revenue or the time spent by users on the platform. We define the expected reward of $\pi_{test}$
\begin{align*} 
R(\pi_{test})&= \mathbb{E}_{x \sim D} \bigg( \mathbb{E}_{a \sim \pi_{test}}(m) \bigg)
\\&= \iiint\limits_{X A m} m q(m|a, x)\pi_{test}(a|x)D(x).
\end{align*}
where $m$ is the reward corresponding to the considered business metric.
\\
\\
To compute an estimator of $\mathbb{E}_{\pi_{test}}[m]$ we have access to tuples $(x_{i}, a_{i}, m_{i})$ that were observed in the production environment, i.e $a_{i}$ was chosen with following $\pi_{prod}$.
\\
\\
\fi
\section{Reducing Estimators Variance}
\label{sec:literatureReview} 
\subsection{Control Variates} 
\label{sec:controlVariates}
Control variates are a popular variance reduction method in statistics. It consists of finding a second random variable with known expectation and which is correlated with the variable to estimate in order to reduce the variance of its estimation.

\paragraph{Doubly robust estimator}
The easiest case is when we dispose of external knowledge like a reward model. We can use this as a control variate to improve our current estimator \cite{dudik2011doubly}. 
We assume we can access a model $\bar{r}(a,x)$ that estimates the expected reward of each action $a$ at context $x$. We define the \emph{doubly robust} estimator:
\begin{align*}
\hat{\cR}^{\rm DR}(\pi_{t}) &=  \sum_{(x,a,r) \in \cS_n} \bigg( \left(r - \bar r(a,x)\right) w(a,x) + \lE_{\pi_t}\left[\bar{r}(A,X) \middle| X=x\right]\bigg).
\end{align*} 
This estimator is unbiased and, if the predicted reward $\bar{r}(A,X)$ is well correlated with the actual reward $R$, it has lower variance than IS (see \cite[\S8.9]{owen2010Monte} for instance).

However, this estimator has several drawbacks in the setting of recommendation systems. 
First, having an accurate model of the reward given the action is challenging when the number of possible actions is very large.
To overcome this problem, \citet{williams1992} -- in the context of reinforcement learning -- proposed to use a model $\bar{r}$ not dependent on the action $a$. It has a higher variance than the initial DR model but avoid the marginalization over all the actions.

However, this approach does not solve the second and biggest drawback of this method: when the reward has a high variance even conditionally to $X$ and $A$, the predicted reward cannot have a strong correlation with it. 
For instance, when the computed metric is the number of clicks -- i.e. the reward $R$ follows a Bernoulli with parameter close to 0 -- the actual expected reward per action (typically around $10^{-3}$ in display advertising for instance) hardly correlates with $R$ ($0$ or $1$). 
In this particular case, DR is very close to IS: the use of the click model does not help to reduce the variance. 
Even if no good model is available, another control variate can be implemented to decrease the variance of IS. 

\paragraph{Normalised importance sampling}
We know that $\mathbb{E}_{\pi_{p}}[W] = 1$. Using the empirical average $\frac{1}{n}\sum_{(x,a,r) \in \cS_n} w(a,x)$ as a global ratio control variate, we have the normalized importance sampling (NIS) estimator \cite{powell1966weighted, Swaminathana}:
\begin{align*} 
\hat{\cR}^{\rm NIS}(\pi_{t}) = \frac{1}{\sum_{(x,a,r) \in \cS_n} w(a,x)} \sum_{(x, a, r) \in \cS_n} w(a,x) r
\end{align*} 
The normalizing constant is equal to the sum of the importance weights and is equal in expectation to n, the number of examples in the dataset. It is a biased estimate of the expected reward but with lower variance than the \textit{basic importance sampling} estimator.
It is a consistent estimator of $\lE_{\pi_t}[r]$ and the bias decreases in $1/n$. Thus NIS, with a certain amount of data is very close to BIS and the variance is not decreased. 
\\
\\
The main problem of such methods aimed at reducing the variance without introducing any bias (at least asymptotically) is that if we do not dispose of a strong external knowledge, we do not reduce the variance that much.

\subsection{Capping weights} 
\label{sec:cappingWeights}

\paragraph{Capped importance sampling} 
Capping weights is another way to control the variance of the IS estimator. Two forms of capping were introduced: max capping and zero capping. 
For some capping value $c>0$, these estimators are respectively defined as: 
\begin{align*}
\hat{\cR}^{\rm maxCIS}(\pi_t, c) =  \frac{1}{n} \sum_{(x,a,r) \in \cS_n} \min(w(a,x),c) r 
\end{align*}
and 
\begin{align*}
\hat{\cR}^{\rm zeroCIS}(\pi_t, c) =  \frac{1}{n} \sum_{(x,a,r) \in \cS_n} \textbf{1}_{w(a,x) < c} w(a,x) r
\end{align*}
In the following, we denote the capped weights by $\bw(a,x)$, for zero capping $\bw(a,x) = \textbf{1}_{w(a,x) < c} w(a,x)$ and for max capping $\bw(a,x) = \min(w(a,x), c)$.
All calculations presented in the following -- except when explicitly specified otherwise -- are valid for both zero capping and max capping.
Intuitively, the behaviour of these two estimators is the same since the capped importance weights are very big compare to the capping parameter: in Fig. \ref{fig:w_distrib}, we provide some empirical results on weights encountered when evaluating recommendation system policies. 
Both capping methods show very similar results and we only report max capping results in the experiments.

However, capping comes at the cost of introducing a bias: when introducing capping on the IS estimator, we only account for a sub-part of $\lE_{\pi_t}[R]$:
\begin{align*}
\lE_{\pi_t}[R] &= \lE_{\pi_t}\left[R \frac{\overline{W}}{W}\right]  + \lE_{\pi_t}\left[R \frac{W-\overline{W}}{W}\right]  \\
& = \underbrace{\lE_{\pi_p}[\hat{\cR}^{\rm CIS}(\pi_t, c)]}_{\cR^{\rm CIS}(\pi_t, c)}
 + \underbrace{\lE_{\pi_t}\left[R \frac{W-\overline{W}}{W}\middle| W>c\right] \lP_{\pi_t}(W>c)}_{\cB^{\rm CIS}(\pi_{t}, c)}
\end{align*}
One of the main issues of only estimating $\cR^{\rm CIS}(\pi_t, c)$ is that the bias term $\cB^{\rm CIS}(\pi_{t}, c)$ becomes low only if $\lE_{\pi_{t}}(R | W >c)$ is low.  It means that $\lE[R|A,X]$ has to be low for all $a$ such that $w(a,x) > c$ -- i.e. for all actions that $\pi_{t}$ chooses much more often than $\pi_{p}$. As our test policy $\pi_{t}$ is usually a trial for improving the current system $\pi_{p}$, it is not really satisfying to have a estimator with a low bias only if $\pi_{t}$ performs poorly on actions it chooses more often than the current system. 
More formally, as we want to take a statistically significant decision, we need to build a confidence interval around $\hat{\cR}^{\rm CIS}(\pi_t, c)$.
We can bound $\cR^{\rm CIS}(\pi_t, c)$ using any concentration bound (e.g. an empirical Bernstein bound \cite{maurer2009empirical, bottou2013counterfactual}). However the bias term can only be controlled in the worst case: $0 \leq \cB^{\rm CIS}(\pi_{t}, c) \leq r_{\max} \left(1-\lP(W\leq c)\right)$. As explained right before, this inequality only gets tight when $r_{\max}$ is lower on the capped volume than elsewhere. 

\subsection{No good practical trade-off for CIS}
\label{sec:no-tradeoff}
In practice, no capping parameter for CIS yields confidence interval small enough to decide whether $\pi_t$ is a better policy than $\pi_p$. The capping parameter used in CIS introduces a bias-variance tradeoff: Increasing its value decreases the bias and increases the variance of the estimator. However, we demonstrated in the previous section that a volume of capped weights too high will lead to an invalid bias-variance tradeoff. Based on experimental results, we show that this critical limit is often reached for recommendation systems used in production. 

We consider an A/B test that was implemented in production and look at the importance weights used to compute the performance of $\pi_t$ based on logs gathered with $\pi_p$. \autoref{fig:w_distrib} shows their distribution according to the test policy. \autoref{fig:bias_variance} shows the variance and the upper bound on the bias (see \autoref{sec:cappingWeights} for definition) depending on the capping parameter. We clearly see the tradeoff between the two measures when the capping parameter changes. In both cases, we can determine the values of the capping parameter for which the variance and bias are lower to the uplift that we want to measure offline (usually, we consider a 1\% uplift). \autoref{fig:bias_variance} shows that no value of the capping parameter satisfies both criteria: a good variance is achieved below $10^2$ whereas a good value for bias requires a capping parameter above $10^{23}$.

This problem led us to design new estimators that model the bias introduced by capping and achieve better bias-variance tradeoff.

\begin{figure}[h!]
\includegraphics[height=0.20\textheight]{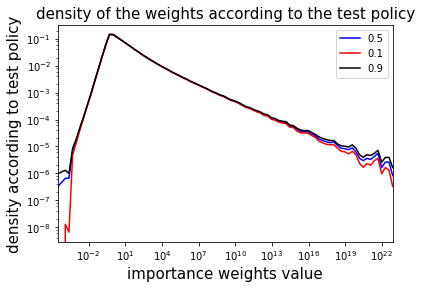}
\caption{Distribution of the importance sampling weights when sampled according to the test policy with 80\% confidence interval. (0.1 corresponds to 10th centile and 0.9 to 90th centile)}
\label{fig:w_distrib}
\end{figure}

\begin{figure}[htb]
\begin{tabular}{cc}
\includegraphics[width=0.22\textwidth, height=0.19\textheight]{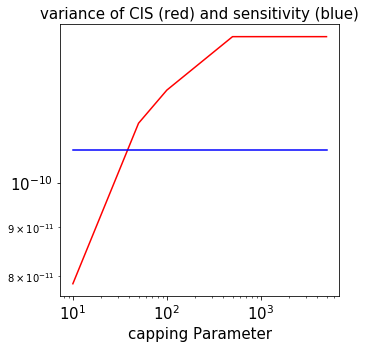} &
\includegraphics[width=0.22\textwidth, height=0.19\textheight]{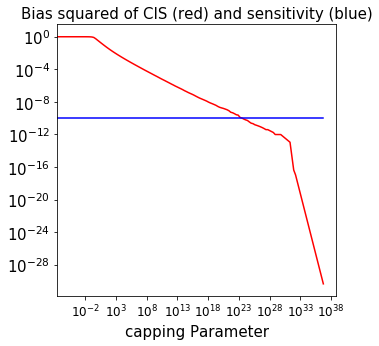}
\end{tabular}
\caption{Variance of CIS, upper bound on the bias of CIS and sensitivity depending on the capping parameter.}
\label{fig:bias_variance}
\end{figure}

\section{Modelling the bias}

As discussed in the previous section, CIS can not reach a tradeoff with low variance and provable low bias. In the following, we present different estimators that model the bias at different scale. We show that the well-known Normalised Capped Importance Sampling provides a model at the global level. Then, we present a new estimator that models the bias at a contextual state level.   	

\subsection{Global bias model}
\label{ssec:ncis}

\paragraph{Normalised Capped Importance Sampling (NCIS)}

A common practice in the literature (e.g. see experiments in \cite{Swaminathana}) is to use the \emph{normalised capped importance sampling} estimator that is defined as: 
\begin{align}
\hat{\cR}^{\rm NCIS}(\pi_t, c) &= \frac{\frac{1}{n}\sum_{(x,a,r) \in \cS_n} \bw(a,x) r}{\frac{1}{n}\sum_{(x,a,r) \in \cS_n} \bw(a,x)}
\label{eq:NCIS}
\end{align}
It involves in re-adjusting the expected reward proportionally to the probability mass capped.
In the following, we show how this estimator models the bias introduced by capping. 
Asymptotically, we can compute the bias of this estimator (the proof is in the appendix):
\begin{align*}
\lP\left( \lim_{n\to\infty} \hat{\cR}^{\rm NCIS}(\pi_t, c, S_n) = \frac{\lE_{\pi_t}\left[\frac{\bW R}{W}\right]}{\lE_{\pi_t}\left[\frac{\bW}{W}\right]}\right)= 1
\end{align*}
with
\begin{align} 
\cR^{\rm NCIS}(\pi_t,c) \triangleq \frac{\lE_{\pi_t}\left[\frac{\bW R}{W}\right]}{\lE_{\pi_t}\left[\frac{\bW}{W}\right]} =  \cR^{\rm CIS}(\pi_t, c) + \cR^{\rm CIS}(\pi_t, c) \frac{1 - \lE_{\pi_{t}}\left[\frac{\bW}{W}\right]}{\lE_{\pi_{t}}\left[\frac{\bW}{W}\right]}
\label{eq:ncis_expect}
\end{align}
Intuitively, if we reuse the expression of $\cB^{\rm CIS}(\pi_{t}, c)$, we see that NCIS, instead of setting $\cB^{\rm CIS}(\pi_{t}, c)$ to zero  like CIS, approximates the performance on the capped volume by the performance on the non-capped volume. 

In the case of zero capping, the approximation made by NCIS is even more intuitive: NCIS makes the assumption that overall, 
\begin{align*} 
\lE_{\pi_{t}}\left[R|W >c\right] \approx  \lE_{\pi_{t}}\left[R|W<c\right]\,.
\end{align*}
and approximates $\cB^{\rm CIS}(\pi_{t}, c)$ as: 
\begin{align*}
\cB^{\rm CIS}\left[\pi_{t}, c\right] \approx \lE_{\pi_{t}}\left[R|W<c\right]\lP_{\pi_{t}}(W> c)
\end{align*}
Those expectations are both on the action $A$ and on the context $X$.

This approximation is exact at least in the trivial case when the reward $R$ is independent from both the context $X$ and the action $A$, and thus of the weight $W$. We can expect the approximation to be reasonable when the noise level is high because in this case the dependency of $R$ on $X$ and $A$ is low.
However we shall see in the next section why this assumption may be poor in practice.

\subsection{Need for a local bias modelling}
\label{counterexample}
The NCIS estimator compensates the capping with a proportional rescaling. However, this rescaling is performed globally, while the underlying data may contain many sub-groups with different average rewards. For instance, the MovieLens dataset \cite{Harper2015} exhibits a small group of frequent user and a large majority of occasional users. In the context of e-commerce, \cite{businessinsider} reports that Prime members spend in average 4.6 times more than non-prime members on Amazon\texttrademark{}. When operating with such different groups of customers, it is quite common to introduce changes in the recommender system that do not equally affect the different groups. For instance, one could decide to favour good deal products for registered customers, which impact both groups differently. As the new policy impact differs from registered to non-registered customers, the capping may be stronger for one of the groups (formalised by the quantity $\bW/W$), thus a blind global rescaling introduces a large bias (see Table \ref{tab:counter_ex} for a toy counter-example).

\begin{table}
  \caption{Counter-example for NCIS. The test policy $\pi_t$ proposes more good deal products to registered customers, improving the (already good) performance on them. The effect is neutral on unknown (low performing) customers. For $\pi_t$, NCIS estimates a reward of 1.8 while the (true) expected reward is 2.1. As the reward of $\pi_p$ is 1.9, we would wrongly conclude that $\pi_t$ is worse than $\pi_p$.}
  \label{tab:counter_ex}
  \begin{tabular}{lcc}
    \toprule
    &\emph{registered customers}&\emph{unknown customers}\\
    \midrule
    Proportions & 0.1 & 0.9 \\
    Performance of $\pi_p$ & 10 & 1\\
    Performance of $\pi_t$ & 12 & 1\\
    $\lE(\bW/W)$ & 0.7 & 1\\
  \bottomrule
\end{tabular}
\end{table}

More formally, from \eqref{eq:ncis_expect}, we can directly deduce that the asymptotic bias of NCIS can be written as
 $$\cB^{\rm NCIS}(\pi_t,c) \triangleq \lE_{\pi_t}[R] - \cR^{\rm NCIS}(\pi_t,c) = - \left.{\rm Cov}_{\pi_t}\left(R, \frac{\overline{W}}{W}\right) \middle/ \lE_{\pi_t}\left[\frac{\overline{W}}{W}\right]\right.,$$ 
 which makes NCIS consistent when $W$ and $R$ are independent. However, as we just discussed, the reward $R$ and the capping $\bW/W$ may be correlated through a third confounding variable such as the type of customer. This information is contained in the context $X$, which leads us to decompose the bias conditionally to $X$,
$$\cB^{\rm NCIS}(\pi_t,c) = - \frac{{\rm Cov}_{\pi_t}\left(\lE[R|X], \lE\left[\frac{\overline{W}}{W}\middle|X\right]\right) + \lE_{\pi_t}\left[{\rm Cov}\left(R, \frac{\overline{W}}{W}\middle| X\right)\right]}{\lE_{\pi_t}\left[\frac{\overline{W}}{W}\right]}$$

\iffalse
In the setting of recommender system, this assumption is quite unrealistic. It is quite probable that there exists a correlation between R and the fact that W is capped, for instance coming from spurious correlations such as the new policy improving the performance over a subset of user with very high performance. 

To illustrate this, we could assume that there exists two types of users: \emph{new} and \emph{recurring} and we built $\pi_t$ to improve over \emph{recurring} users and acting identically on \emph{new} ones (see Table \ref{tab:counter_ex} for summary). Thus, the estimation of the performance of \emph{new} users is accurate (the two policies behave identically) while the estimation of the \emph{recurrent} users is noisy. In such case, NCIS is going to model the performance of the noisy \emph{recurrent} users (high performance) by the one of the well estimated \emph{new} users (low performance) and conclude that $\pi_t$ is worse than $\pi_p$ while in reality it is better for all types of users.
\fi

To improve over NCIS, we can assume the first term of the numerator to be dominant. Indeed, in practice the recommendation (the action $A$) itself is usually an unsolicited attempt to influence a user with a pre-existing intent (the context $X$). Thus, one can realistically expect the reward $R$ to be more correlated to the initial intent than to the recommendation.
\iffalse
$$\left|{\rm Cov}_{\pi_t}\left(\lE[R|X], \lE\left[\frac{\overline{W}}{W}\middle|X\right]\right)\right| \gg \left|\lE\left[{\rm Cov}_{\pi_t}\left(R, \frac{\overline{W}}{W}\middle| X\right)\right]\right|$$
Thus, the bias of NCIS is completely dominated by it's first term:
 $$\cB^{\rm NCIS}(\pi_t,c) = - \frac{{\rm Cov}_{\pi_t}\left(\lE[R|X], \lE\left[\frac{\overline{W}}{W}\middle|X\right]\right) + \lE\left[{\rm Cov}_{\pi_t}\left(R, \frac{\overline{W}}{W}\middle| X\right)\right]}{\lE\left[\frac{\overline{W}}{W}\right]}$$ 
 \fi
Following this idea, we propose next several ways to build "local" versions of the NCIS estimator -- by normalising conditionally to the context -- to get rid of the first term of the bias (the biggest) and obtain estimators with a much smaller bias.

\subsection{Piecewise constant model}
\label{sec:stratification}
The following estimator is the first step toward a model of the bias at a finer scale. The simplest way to build a local version of NCIS is to use stratification. In other words, to make a piecewise version of NCIS. The goal is to find a partition $\cG$ of $\cX$ in order to use the decomposition of the expectation over this partition,
\begin{align*}
\lE_{\pi_{t}}[R]  =  \sum_{g\in\cG} \lE_{\pi_{t}}[R | X \in g] \lP( X \in g)
\end{align*}
and then estimate the expectation separately on each group of the partition with the NCIS estimator:
\begin{align*}
\hat{\cR}^{\rm PieceNCIS}(\pi_t, c) = &  \sum_ {g\in\cG} \alpha_g \hat{\cR}|^{\rm NCIS}_g(\pi_t, c)
\end{align*}
where $\alpha_g = \sum_{(x,a,r) \in \cS_n}  \textbf{1}_{x \in g} / n$ estimates $\lP( X \in g)$ and $\hat{\cR}|^{\rm NCIS}_g(\pi_t, c)$ is the restriction of the NCIS estimator to the group $g$ to estimate $\lE_{\pi_{t}}[R | X \in g]$:
\begin{align*}
\hat{\cR}|^{\rm NCIS}_g(\pi_t, c) = &  \frac{\sum_{(x,a,r) \in \cS_n} \textbf{1}_{x\in g}\bw(a,x) r}{\sum_{(x,a,r) \in \cS_n} \textbf{1}_{x\in g}\bw(a,x)}
\end{align*}
A desirable constraint on $\cG$ is to be independent from the tested policy $\pi_t$. 
When removing the capping -- i.e. $c \to \infty$ -- the estimator $\hat{\cR}^{\rm PieceNCIS}(\pi_t, c)$ is consistent.

Apart from this constraint, the partition can be constructed from any hand-crafted splits on features of $x$. 
However, even though easy in practice, this option is not satisfying because the performance of the estimator will strongly depend on a manual choice of the partition.
A more agnostic method to build an empirically "good" partition, would be to learn a value function $V(x)$ to predict the expected reward given a context $x$ and build the partition based on the output of this model. For instance, in the experiments presented in Section \ref{sec:experiments}, we use a regular partition of the output of the model in the log-space (base $b$):
\begin{align*}
\cG = \left\{V^{-1}(I_k): I_k = [b^k, b^{k+1}], k\in \lZ\right\}
\end{align*}
This approach provides two advantages: 1) given a group, the reward does not depend strongly on $x$ anymore, which was the first objective of stratification, 2) the size of the partition is more controlled than with a hand-crafted one, leading to more samples per group and thus less estimation problems.
However, it comes at the cost of needing to fit a value function on a separate set of data.

\subsection{Pointwise model}
\label{sec:apprTest}
To avoid having to learn and design a value model to perform a stratification, we push the idea further and use the decomposition:
\begin{align*}
\lE_{\pi_{t}}[R]  =  \sum_{x\in\cX} \lE_{\pi_{t}}[R | X = x] \lP( X = x)
\end{align*}

However, building an estimator of $\lE_{\pi_{t}}[R | X = x]$ becomes more challenging.
Following the underlying idea of NCIS, we want to make the following approximation:
\begin{align*}
\lE_{\pi_{t}}[R | X = x] \approx \frac{\lE_{\pi_{t}}\left[R\frac{\overline{W}}{W}\middle| X = x\right]}{\lE_{\pi_{t}}\left[\frac{\overline{W}}{W} \middle| X = x\right]}
\end{align*}
Unfortunately, when conditioning on a value of $x$, we cannot use a simple ratio estimator such as 
$$\frac{\sum_{(x',a,r) \in \cS_n} \textbf{1}_{x'=x}\bw(a,x') r}{\sum_{(x',a,r) \in \cS_n} \textbf{1}_{x'=x}\bw(a,x')}$$
Indeed, the number of samples in the training set exactly matching a given value of $x$ is very low (it can even be 0 if $x$ is continuous), so the bias of the ratio estimator is not negligible anymore.
Fortunately, when conditioned on $x$, we can be much better at the estimation of $\lE_{\pi_{t}}\left[\frac{\overline{W}}{W} \middle| X = x\right]$. 
In fact, we could even compute it exactly by a simple marginalisation on the actions, but the number of actions is too large in our case to perform this computation in a reasonable time.
Moreover, we are not aware of any closed-form for it, even when both policies follow simple models (e.g. Plackett-Luce or Mallows).

However, we can efficiently sample from the policy $\pi_t$, and therefore compute a Monte Carlo approximation of this probability. 
Since we actually want an estimator of the ratio ${1}/{ \lE_{\pi_{t}}\left[\frac{\overline{W}}{W} \middle| X = x\right]} $, we can use a rejection sampling technique such as Lahiri's \cite{lahiri1951method} or Midzuno-Sen  method \cite{midzuno1951sampling, sen1952present} to get an unbiased estimate denoted $\hat{IP}_c(x)$. 
In practice to build an estimator of ${1}/{ \lE_{\pi_{t}}\left[\frac{\overline{W}}{W} \middle| X = x\right]} $, we use the Midzuno-Sen method. We define a random variable $U$ uniformly distributed between 0 and 1. Then, we do successively
\begin{itemize} 
\item sample $u$ from the uniform distribution and $w_1$ from $\pi_t$ until reaching $w_1 < u$
\item sample $w_2,...,w_n$ from $\pi_t$ 
\item return $\frac{n}{\frac{\overline{w_1}}{w_1}+...+\frac{\overline{w_n}}{w_n}})$
\end{itemize}
Through this method, the expectation of $n/(W_1/\overline{W_1}+...+W_n/\overline{W_n})$ is equal to ${1}/{ \lE_{\pi_{t}}\left[\frac{\overline{W}}{W} \middle| X = x\right]}$. In the following, we will note it $\hat{IP}_c(x)$.

Finally, we can define the following estimator:
\begin{align*}
\hat{\cR}^{\rm PointNCIS}(\pi_{t},c) = \frac{1}{n}\sum_{(x,a,r) \in \cS_n} \hat{IP}_c(x)  \bw(a,x) r
\end{align*}
We need to notice one pitfall of this method: if the test distribution and the prod distribution are really dissimilar, the expectation $ \lE_{\pi_{t}}\left[\frac{\overline{W}}{W} \middle| X = x\right] $  may become very low. It means that the effective weight by which we multiply the reward $r$, $ \frac{ w }{ \lE_{\pi_{t}}\left[\frac{\overline{W}}{W} \middle| X = x\right] }$ might be much bigger than the capping value, and the estimator would have a high variance. We can overcome this problem when using max capping by decreasing the capping value when the distributions do not overlap enough (see appendix for justification).

\begin{table}[h!tb]
\centering
\caption{Summary table of the different estimators. First column sum up the formulae of the estimators, second one the approximation $\tilde{\cB}$ of the bias term $\cB$ in the general case and the case of zero capping}
\label{my-label}
\bgroup
\def\arraystretch{1.5}
\setlength\tabcolsep{0.4em}
\begin{tabular}{lcc} 
\hline
 &   $\hat{\cR}(\pi_t,c)$ & approx $\tilde{\cB}^{\rm CIS}(\pi_t,c,x)$  \\ \hline
CIS & $\frac{1}{n}\sum_{S_n} r \overline{w}(a,x)$ & $0$  \\
NCIS & $\frac{\sum_{\cS_n} r \overline{w}(a,x)}{\sum_{\cS_n} \overline{w}(a,x)} $ & $\lE_{\pi_t}\left[\frac{R\overline{W}}{W}\right] \!\frac{1 - \mathbb{E}_{\pi_{t}}\left[\frac{\bW}{W}\right]}{\mathbb{E}_{\pi_{t}}\left[\frac{\bW}{W}\right]}$ \\
PieceNCIS & $\sum\limits_{g\in\cG} \alpha_g \hat{\cR}|^{\rm NCIS}_g(\pi_t, c)$ & $\lE_{\pi_t}\!\!\left[\frac{R\overline{W}}{W}\middle | X\!\in\! g \right] \!\frac{1 - \mathbb{E}_{\pi_{t}}\left[\frac{\bW}{W}\middle| X\in g\right]}{\mathbb{E}_{\pi_{t}}\left[\frac{\bW}{W}\middle| X\in g\right]}$ \\
PointNCIS  & $\frac{1}{n}\sum\limits_{\cS_n} \hat{IP}_c(x)  \bw(a,x) r$ &  $\lE_{\pi_t}\!\!\left[\frac{R\overline{W}}{W}\middle | X\!=\!x \right] \!\frac{1 - \mathbb{E}_{\pi_{t}}\left[\frac{\bW}{W}\middle| X=x\right]}{\mathbb{E}_{\pi_{t}}\left[\frac{\bW}{W}\middle| X=x\right]}$ \\\hline
\end{tabular}
\egroup
\end{table}
\section{Experiments}
\label{sec:experiments}
We use the A/B test history of a commercial recommender system to compare the uplifts $\Delta \hat R$ estimated by our \emph{offline A/B tests} methods with the ground truth $\Delta R$ estimated during the \emph{online A/B tests}.

\label{sec:results}
\begin{figure}[h!]
\includegraphics[height=0.25\textheight]{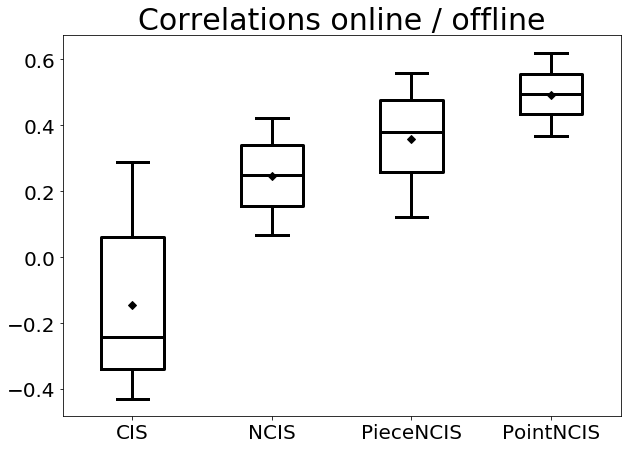}
\caption{Correlation between online and offline uplifts. Confidence bounds are obtained using bootstraps. Whiskers are $10\%$ and $90\%$ quantiles and the boxes represent quartiles.}
\label{fig:correlation}
\end{figure}

\subsection{Dataset}

We have access to a proprietary dataset of 39 \emph{online A/B tests}, representing a total of few hundreds of billions of recommendations.
We consider a click-based business metric. Since clicks are relatively rare, the reward signal has a high variance, even conditioned to the context and to the action.
For each test, we consider: 
\begin{itemize}
\item the two policies involved in the A/B test $\pi_p$ and $\pi_t$,
\item for each display (units $x$ whose features represent the context of the display and past interactions with the user), we have access to,
\begin{itemize}
\item the top-$K$ ranking (action $a$) chosen along with its probability under the logging policy -- $\pi_p(a|x)$ on control population A or $\pi_t(a|x)$ on test population B,
\item the set of eligible items,
\item the observed reward $r$.
\end{itemize}
\end{itemize} 

On each A/B test, we can compute the online estimate of the uplift $\Delta \hat \cR(\pi_p, \pi_t) = \hat \cR(\cS_t) - \hat \cR (\cS_p)$ -- our ground truth -- and for any estimator, the offline estimate of the uplift $\hat \cR^{\rm EST}(\pi_t, c, \cS_p) - \hat \cR (\cS_p)$. We added ${\cS_p}$ in the arguments of the estimator to emphasize the estimator is computed on data collected by the \emph{prod} population.

We could run several \emph{offline A/B tests} based on the data on a single \emph{online A/B test} -- e.g. $\Delta \hat \cR(\pi_p, \pi_t)$ on the data of control population $A$ or $\Delta \hat \cR(\pi_t, \pi_p)$ on the data of test population $B$. We only keep one of them, because if we compare them with the same \emph{online A/B test} result, the comparisons wouldn't be independent. In the following, we consider $\Delta \hat \cR(\pi_p, \pi_t)$ on the data of control population $A$, which means the control (logging) policy is the production one $\pi_p$ and the tested one is $\pi_t$.

\subsection{Estimators}

We compare here four estimators presented in the previous sections: CIS, NCIS, PieceNCIS and PointNCIS. We set the capping value to $c=100$ based on the graph presented in Section \ref{sec:no-tradeoff}.
We discarded non-capped estimators such as IS or NIS due to their very high variance: the confidence intervals on $\Delta \hat R$ would never lead to a positive or negative decision, it would always be neutral.
Moreover, we chose to ignore the doubly-robust estimator (DR) for two reasons. Indeed if the estimator is not capped, it suffers the same issue as IS and NIS. When it is capped \cite{thomas2016data}, its performance strongly depends on the reward model such that the optimal policy under such estimator when $c \to 0$ is the deterministic policy choosing the argmax action on the reward model.

\paragraph{Computation time}
The NCIS and PieceNCIS estimators need to be computed on the entire dataset, on both positive and negative examples. 
Like CIS, PointNCIS just needs to read the positive examples : it leads to a huge gain in computation time and efficiency when the reward is very sparse.

\begin{figure}
\includegraphics[height=0.28\textheight]{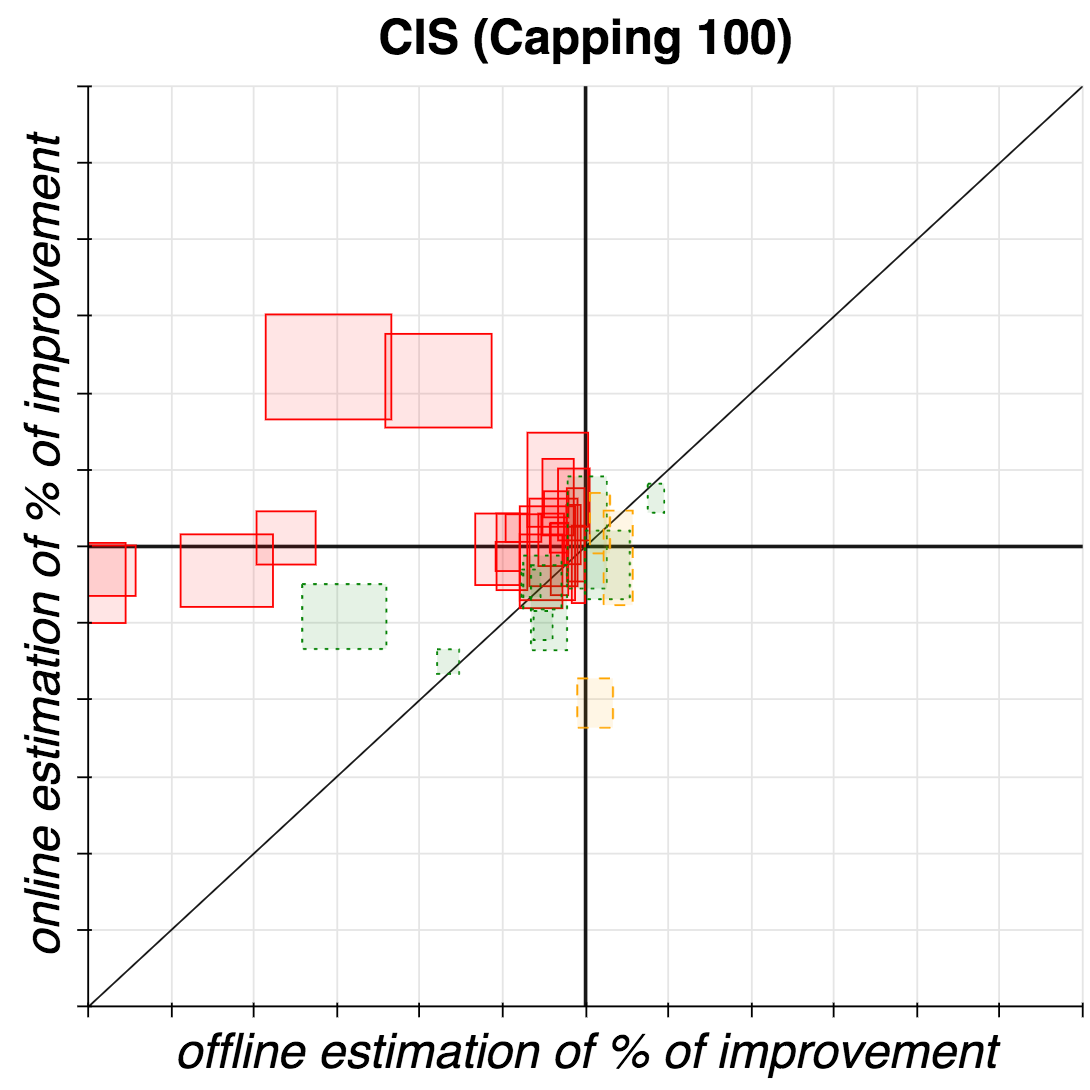}
\includegraphics[height=0.28\textheight]{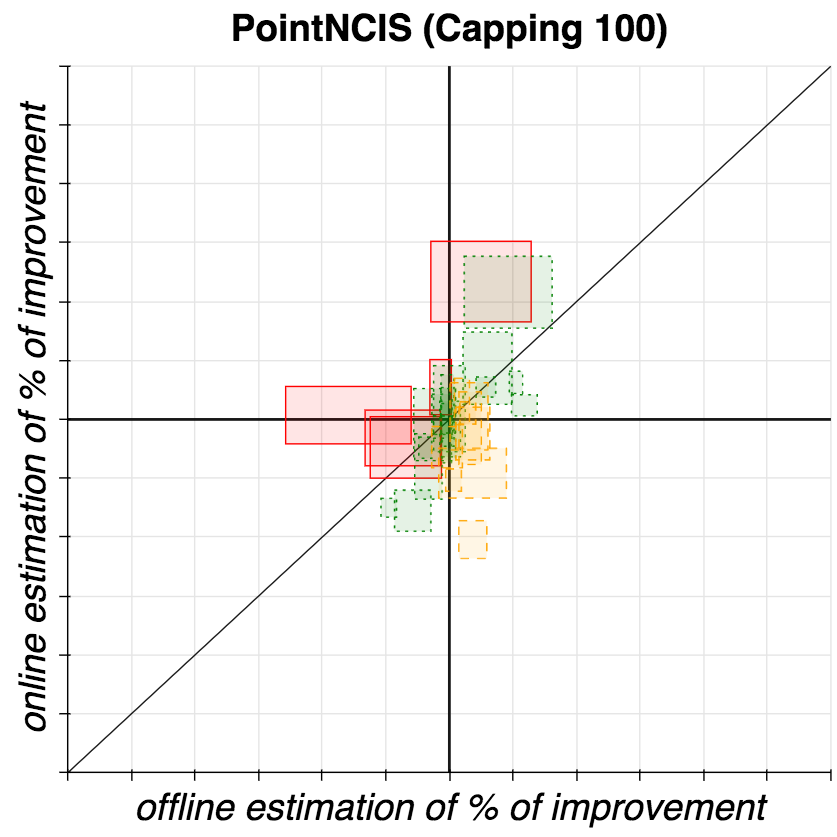}
\caption{Comparison of online/offline decision. A box is an A/B test. The width (resp. height) of the box is a $90\%$ confidence interval on the offline (resp. online) uplift. The scale is the same for both axis. Green/dotted: right decision. Orange/dashed: false positive). Red/plain: false negative.}
\label{fig:ooboxes}
\end{figure}

\subsection{Correlation in Online / Offline Uplifts}

To compare the performance of the different estimators, we first simply compute the correlations between the values of $\Delta \hat R$ and $\Delta R$ on the series of A/B tests. Results are presented in Fig. \ref{fig:correlation}. As expected, CIS performs quite poorly, due to the strong capping. Compensating the bias only globally with NCIS already proves to be a good improvement. Then, the more local the model on the bias, the better performing the estimator: the piecewise estimator is better than the global and the pointwise is better than the piecewise.

While these first results are already satisfying, it does not reflect the fact that all types of errors are not as bad and that the different estimators do not lead to the same types of errors -- which, as we will see in the following, explains why the correlation online/offline of CIS seems to be negative.

\subsection{False positive VS False negative rate}

From the point of view of continuously improving a production recommender system, false negatives are much worse than false positives. A false positive -- predicted positive by the offline estimator and actually negative or neutral during the A/B test -- has a cost limited to the A/B test duration. It slows down the pace of improvement of the system. On the contrary, a false negative is an actual improvement of the system that will never be tested online due to a mistake in the offline estimation. Such an error has a long term cost, as it is an actual improvement that won't be rolled-out.

In Figure \ref{fig:ooboxes}, we show the comparison between online and offline decisions for our two extreme estimators CIS and PointNCIS.
The number of false negatives of CIS is especially interesting: it reflects that CIS always underestimates the reward. It explains why the correlation of CIS is negative in Fig.~\ref{fig:correlation}, just because there are more positive decisions than negative in our dataset.

To have a clearer understanding of the quality of the different estimators, we sum up in Table \ref{tab:performance} the uplift correlations (same data as in Fig. \ref{fig:correlation}) along with metrics on the decision (positive / neutral / negative) such as the precision and the false negative rate (FNR). 
We can split the improvement in three. 
First, only using a global model of the bias -- such as NCIS -- is enough to drastically reduce the number of false negative errors: the FNR goes from 0.64 (CIS) to 0.33 (NCIS). This is also reflected by an improvement in precision and correlation.
Then, another improvement comes from using more local approximations of the bias: precision and correlation also improves from NCIS to PieceNCIS or PointNCIS. 
When looking more carefully at the changes of decisions between the estimator, we actually noticed several A/B tests that include changes similar to the counter-example presented in Sec. \ref{counterexample}, which support the need for more local estimators  and explain the improvement the more local estimators show.
Finally, we can also observe a better correlation of PointNCIS over PieceNCIS. However, as the precision is quite similar, it may only come from a better estimation of the value of the uplift and not from more aligned decisions. 
In the end, its performance and computation efficiency improvements make PointNCIS a better choice than the other estimators.

\begin{table}[h!tb]
\centering
\caption{Performance of the different estimators. Uncertainty is reported based on $10\% / 90\%$ confidence intervals obtained by bootstrapping. \emph{CI size} indicates the average relative width of the confidence interval compared to CIS one.}
\label{tab:performance}
\begin{tabular}{lcccc}
\hline
 &   Correlation & Precision & FNR & CI size \\ \hline
CIS &$-0.15\pm0.35$ & $0.28\pm0.10$ & $0.64\pm0.11$ & $ - $ \\
NCIS & $0.24\pm0.18$ & $0.47\pm0.11$ & $0.33\pm0.11$ & $1.1$ \\
PieceNCIS & $0.36\pm0.22$& $0.53\pm0.11$  & $0.28\pm0.08$ & $1.5$ \\
PointNCIS  & $0.49\pm0.13$& $0.56\pm0.13$  & $0.16\pm0.09$ & $0.8$ \\ \hline
\end{tabular}
\end{table}

\section{Conclusion}
Through the paper, we exhibited the different sub-optimality properties of the traditional counterfactual estimators in the setting of recommender systems. We introduced several new estimators exhibiting a better bias-variance trade-off than the traditional Normalised Capped Importance Sampling estimator. Then, we provided a benchmark of the different offline estimators with experiments conducted on a large commercial recommender system. In the future, we plan to investigate other ways to reduce the variance of the offline estimators. A simple way would be to exploit that several recommendation policies are implemented every week: combining them in order to make the mixture of policies closer to the test policy could help to reduce the variance of the different offline estimators.  

\appendix
\section{APPENDIX}
\subsection{Analysis of the bias of the NCIS estimator}
In this section, we analyse the bias of the Normalised Capped Importance Sampling estimator in the general case (true for both max and zero capping). We also show how the formula can be simplified in the case of zero capping. 

\begin{lemma}[Asymptotical behaviour of $\hat{\cR}^{\rm NCIS}$]
Let $\hat{\cR}^{\rm NCIS}$ the normalised capped importance sampling estimator defined in \eqref{eq:NCIS}. Then, 
\begin{align*} 
\mathbb{P}\bigg( \lim_{n\to\infty} \hat{\cR}^{\rm NCIS}(\pi_t, c, \cS_n) = \hat{\cR}^{\rm CIS} + \hat{\cR}^{\rm CIS} \frac{1 - \lE_{\pi_{t}}[\frac{\bW}{W}]}{\lE_{\pi_{t}}[\frac{\bW}{W}]}\bigg) = 1
\end{align*} 
\end{lemma}
\begin{proof} 
First, we study the convergence of the numerator: It is the mean of n i.i.d. random variables. Thus, according to the strong law of large numbers, with probability one, 

\begin{align*} 
\lim_{n\to\infty} \frac{1}{n}\sum_{(x,a,r) \in \cS_n} \bw(a,x) r &= \lE_{\pi_{p}}\left[\bW R\right] = \lE_{\pi_{t}}\left[\frac{\bW R}{W}\right]
\end{align*} 

The denominator is also the mean of n i.i.d. random variables. Then, with probability one, 

\begin{align*} 
\lim_{n\to\infty} \frac{1}{n}\sum_{(x,a,r) \in \cS_n} \bw(a,x) &= \lE_{\pi_{p}}\left[\bW\right] = \lE_{\pi_{t}}\left[\frac{\bW}{W}\right]
\end{align*} 

Hence, we reach
\begin{align*} 
 \lim_{n\to\infty} \hat{\cR}^{\rm NCIS}= \frac{\lE_{\pi_{t}}\left[\frac{\bW R}{W}\right]}{\lE_{\pi_{t}}\left[\frac{\bW}{W}\right]} 
  &=\lE_{\pi_{t}}\left[\frac{\bW R}{W}\right] + \lE_{\pi_{t}}\left[\frac{\bW R}{W}\right] \frac{1 - \lE_{\pi_{t}}\left[\frac{\bW}{W}\right]}{\lE_{\pi_{t}}\left[\frac{\bW}{W}\right]}
 \\& =\hat{\cR}^{\rm CIS} + \hat{\cR}^{\rm CIS} \frac{1 - \lE_{\pi_{t}}\left[\frac{\bW}{W}\right]}{\lE_{\pi_{t}}\left[\frac{\bW}{W}\right]}
\end{align*} 
\end{proof} 
We observe that NCIS is correcting CIS by approximating the performance on the capped volume by the performance on the non-capped volume. It helps to reduce the bias of CIS. 
We study now the particular case of zero capping. 
\begin{lemma}[Asymptotical behaviour of $\hat{\cR}^{\rm NCIS}_{\rm zero}$]: 
Let $\hat{\cR}^{\rm NCIS}_{\rm zero}$ the capped normalised importance sampling estimator. Then, 
\begin{align*} 
\lP\bigg( \lim_{n\to\infty} \hat{\cR}^{\rm NCIS}_{\rm zero}(\pi_t, c, \cS_n) = \lE_{\pi_{t}}\left[R \textbf{1}_{W\leq c} \right]\bigg) = 1
\end{align*} 
\end{lemma}

\begin{proof} 
Straightforward application of the previous lemma
\end{proof} 
This analysis can obviously be extended to PieceNCIS.

\subsection{Analysis under varying capping parameter} 
To prove that we can control the variance of PointNCIS even though the value of $\tilde{w}_c(a,x) = \frac{ \overline{w}(a,x) }{ \lE_{\pi_t}\left[ \frac{\overline{W}}{W}  \middle| X = x \right]}$ may be higher than the capping $c$, we need to prove that for any value of $c > 1$, we can find a $\tilde{c}$ such that $\tilde{w}_{\tilde{c}}(a,x) \leq c$.
The following lemma states that it is possible to ensure this with max-capping.

\begin{lemma}[max-capping]
For any $a$ and $c > 1$, there exists $\tilde{c}$ such that $\tilde{w}_{\tilde{c}}(a,x) \leq c$.
\end{lemma}

\begin{proof}
For any $\tilde{c}$, $\lE_{\pi_t}\left[ \frac{\overline{W}}{W}  \middle| x \right] = \lE_{\pi_p}\left[\min(W, \tilde{c})\middle| x\right]$, thus
\begin{align*} 
\tilde{w}_{\tilde{c}}(a,x) \leq \frac{\tilde{c}}{\lE_{\pi_p}\left[\min(W, \tilde{c}) \middle| x \right]} \leq \frac{1}{\lP_{\pi_p}(W \geq \tilde{c} | x)} \xrightarrow{\tilde{c} \rightarrow 0} 1 
\end{align*}
as $\lE_{\pi_p}\left[\min(W,\tilde{c}) \middle| x \right] = \tilde{c} \lP_{\pi_p}(W \geq \tilde{c}|x) + \lE_{\pi_p}\left[W[W <\tilde{c}] \middle| x \right] $.
\end{proof}

Unfortunately, we cannot ensure such property for zero-capping, which prevents us from adapting $\tilde{c}$ for PointNCIS.

\begin{lemma}[zero-capping]
There exists $a$, $c > 1$ such that for any $\tilde{c} > 0$, $ \tilde{w}_{\tilde{c}}(a,x) > c$.
\end{lemma}
\begin{proof}
We use a counter-example where two actions $a_0$ and $a_1$ are taken with probabilities $\pi_p(a_0) = p$ and $\pi_t(a_0) = 1 - \pi_p(a_0)$ where $p>0.5$. Then, $w(a_0) = \frac{1-p}{p} < w(a_1) =  \frac{p}{1-p}$.
\begin{description}
\item{Case $\tilde{c} < w(a_0)$:} $\tilde{w}_{\tilde{c}}(a,x)$ is undefined.
\item{Case $\tilde{c} \in \left[w(a_0), w(a_1)\right]$:} $\tilde{w}_{\tilde{c}}(a_1,x) = 0$ and $\tilde{w}_{\tilde{c}}(a_0,x) = \frac{1-p}{p^2}$.
\item{Case $\tilde{c}\geq w(a_1)$:} $\tilde{w}_{\tilde{c}}(a_1,x) = \frac{p}{1-p} > 1$ and $\tilde{w}_{\tilde{c}}(a_0,x) = \frac{1-p}{p}$.
\end{description}
So, if $0.5 < p < \frac{-1 + \sqrt{5}}{2}$ and $1 < c < \min\left(\frac{p}{1-p}, \frac{1-p}{p^2}\right)$ then there exists an action in each case such that $\tilde{w}_{\tilde{c}} > c$.
\end{proof}

\subsection{Additional figures}
\begin{figure}[h]
        \includegraphics[width=0.23\textwidth]{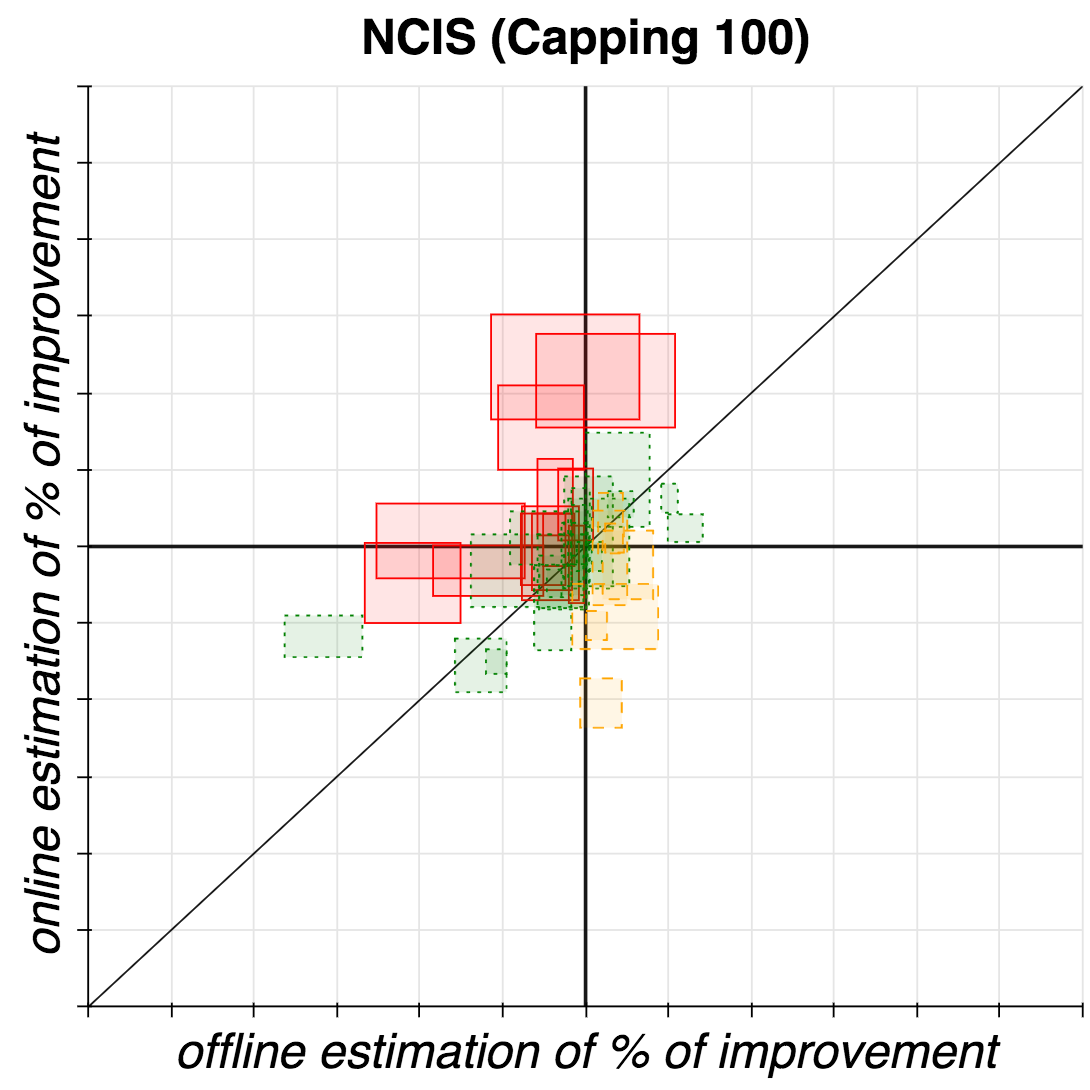}
        \includegraphics[width=0.23\textwidth]{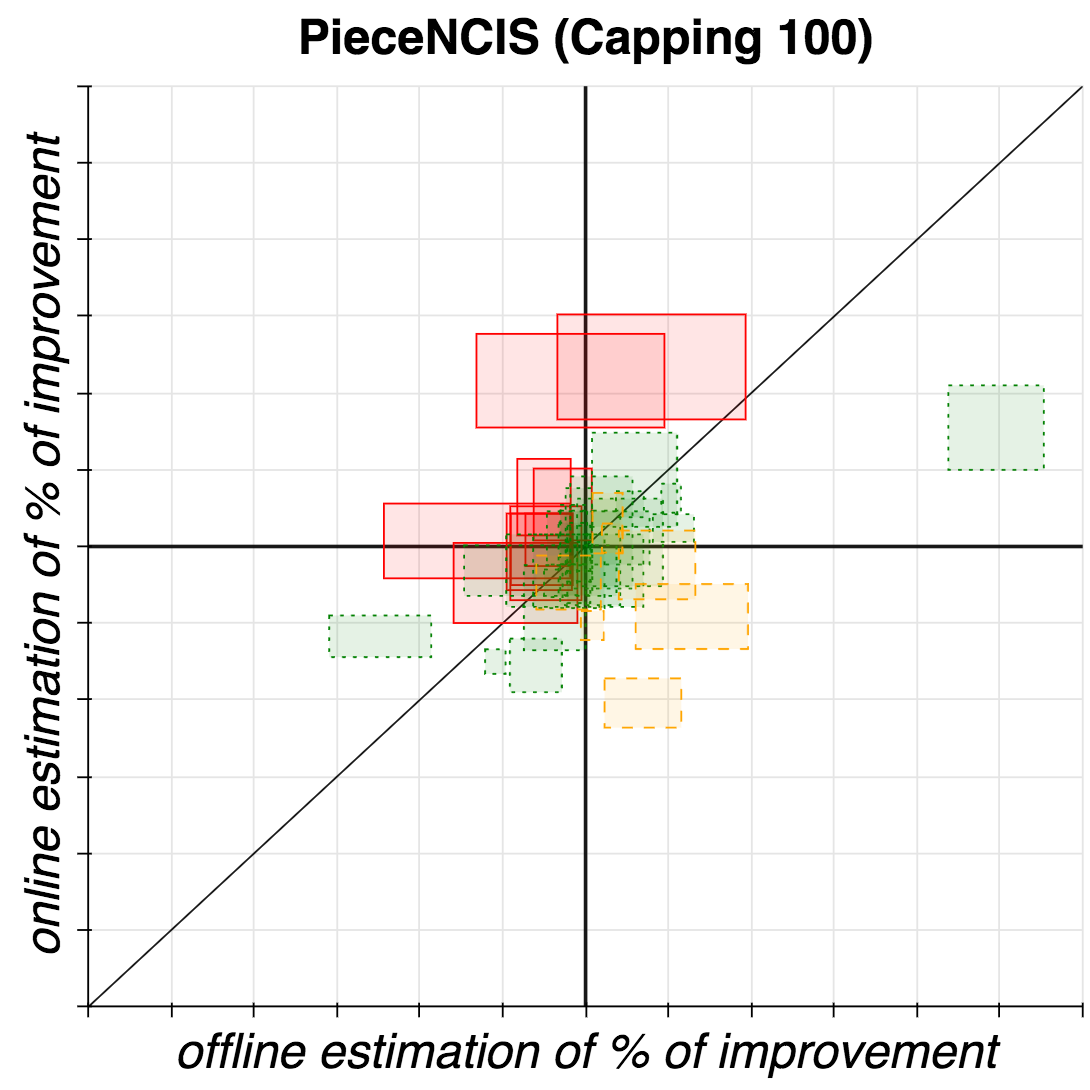}
    \caption{Comparison of online/offline decision. A box is an A/B test. The width (resp. height) of a box is a $90\%$ confidence bound on the offline (resp. online) uplift. The scale is the same for both axis. Green/dotted: right decision. Orange/dashed: false positive. Red/plain: false negative.}
\end{figure}
\newpage
\bibliographystyle{ACM-Reference-Format}
\balance
\bibliography{literature} 

\end{document}